\documentclass[final,12pt]{colt2023}

\title[Universal Online Optimization in Dynamic Environments]{Universal Online Optimization in Dynamic Environments \\ via Uniclass Prediction}
\usepackage{times}
\coltauthor{%
 \Name{Arnold Salas} \Email{asalas@autarky.ai}\\
 \addr Autarky.ai, Greenwich, CT, USA
}

\DeclareMathOperator*{\argmin}{arg\,min}

\newcommand{\yrcite}[1]{\citeyearpar{#1}}
\renewcommand{\cite}[1]{\citep{#1}}

\usepackage[capitalize,noabbrev]{cleveref}
\usepackage{booktabs} % for professional tables
\usepackage{algorithm}
\usepackage{algorithmic}

\newtheorem{assumption}[theorem]{Assumption}

\begin{document}

\maketitle

\begin{abstract}%
Recently, several universal methods have been proposed for online convex optimization which can handle convex, strongly convex and exponentially concave cost functions simultaneously. However, most of these algorithms have been designed with \emph{static} regret minimization in mind, but this notion of regret may not be suitable for changing environments. To address this shortcoming, we propose a novel and intuitive framework for universal online optimization in dynamic environments. Unlike existing universal algorithms, our strategy does not rely on the construction of a set of experts and an accompanying meta-algorithm. Instead, we show that the problem of dynamic online optimization can be reduced to a uniclass prediction problem \cite{crammer06}. By leaving the choice of uniclass loss function in the user's hands, they are able to control and optimize dynamic regret bounds, which in turn carry over into the original problem. To the best of our knowledge, this is the first paper proposing a universal approach with state-of-the-art dynamic regret guarantees \emph{even for} general convex cost functions.
\end{abstract}


\section{Introduction}

Online convex optimization (OCO) has become a leading framework for iterative decision-making, and is able to model various real-world problems such as online routing and spam filtering \cite{hazan22}. OCO can be viewed as a repeated game between a learner and the environment. At each round $t$, the learner chooses an action $\mathbf{x}_{t}$ from a convex set $\mathcal{X}$. Thereafter, the environment selects (possibly in an adversarial manner) a convex cost function $f_{t}:\mathcal{X} \to \mathbb{R}$ and reveals it to the learner. Subsequently, the learner incurs an instantaneous cost of $f_{t}(\mathbf{x}_{t})$.

The standard performance metric is \emph{static} regret which benchmarks the learner's cumulative cost against that of the best fixed action in hindsight:
\begin{equation}
\label{eq:sreg}
\text{S-Regret}_{f}(\mathbf{x}_{1:T})
\equiv \sum_{t=1}^{T} f_{t}(\mathbf{x}_{t}) - \min_{\mathbf{x} \in \mathcal{X}} \sum_{t=1}^{T} f_{t}(\mathbf{x}),
\end{equation}
where the notation $\mathbf{x}_{1:T}$ is defined in \cref{notation}.
In the literature, there are plenty of algorithms for minimizing static regret \cite{cesa-bianchi06, shalev-shwartz11, hazan22}.
However, many interesting real-world applications are characterized by dynamic environments \cite{zhou22}, for which static regret is not an appropriate metric, as it incentivizes static behavior. As a result, stronger notions of regret have been proposed in the literature, including dynamic regret \cite{zinkevich03} and adaptive regret \cite{hazan07b, daniely15}. Throughout this paper, we shall focus on a measure known as \emph{worst-case dynamic regret}, or dynamic regret for short\footnote{This is a special case of the concept of \emph{general dynamic regret}, given by $\sum_{t=1}^{T} f_{t}(\mathbf{x}_{t}) - \sum_{t=1}^{T} f_{t}(\mathbf{u}_{t})$, for any sequence of comparators $\mathbf{u}_{1}, \ldots, \mathbf{u}_{T} \in \mathcal{X}$ \cite{zinkevich03}.}:
\begin{equation}
\label{eq:dreg}
\text{D-Regret}_{f}(\mathbf{x}_{1:T})
\equiv \sum_{t=1}^{T} f_{t}(\mathbf{x}_{t}) - \sum_{t=1}^{T} \min_{\mathbf{x} \in \mathcal{X}} f_{t}(\mathbf{x}).
\end{equation}

There is a plethora of studies on dynamic regret minimization such as \cite{zinkevich03, hall13, besbes15, jadbabaie15, yang16, mokhtari16, zhang17, zhang18b, baby19, zhang20, zhao21}.
However, these studies can only handle \textbf{one} type of convex cost functions. Furthermore, when it comes to strongly convex or exponentially concave (abbr. exp-concave) functions, the respective moduli need to be known or estimated. Needless to say, this lack of universality hinders their applicability.

On the other hand, there are universal algorithms that attain optimal static regret \cite{bartlett07, vanerven16, wang19, wang20, zhang22} or adaptive regret \cite{zhang21} for multiple types of convex cost functions simultaneously. This line of research motivated us to wonder whether it is possible to design a single algorithm that minimizes the dynamic regret for any type of cost functions. We provide an affirmative answer by observing that the problem of dynamic regret minimization can be reduced to a uniclass prediction problem, provided the original cost functions are Lipschitz continuous. Our procedure consists of three steps, namely:
\begin{itemize}
\item First, we show that the Lipschitz continuity of the cost functions trivially implies that
\begin{equation}
\label{eq:dreg-ub}
\text{D-Regret}_{f}(\mathbf{x}_{1:T})
\leq K_{f}\sum_{t=1}^{T}\left\Vert\mathbf{x}_{t} - \Pi_{\mathcal{F}_{t}^{*}}(\mathbf{x}_{t})\right\Vert,
\end{equation}
where $K_{f}$ denotes the Lipschitz constant and
\begin{equation*}
\Pi_{\mathcal{F}_{t}^{*}}(\mathbf{x}_{t})
= \argmin_{\mathbf{x}_{t}^{*} \in \mathcal{F}_{t}^{*}} \ \left\Vert\mathbf{x}_{t} - \mathbf{x}_{t}^{*}\right\Vert,
\quad \mathcal{F}_{t}^{*} \equiv \argmin_{\mathbf{x} \in \mathcal{X}} f_{t}(\mathbf{x}).
\end{equation*}
Using $\Pi_{\mathcal{F}_{t}^{*}}(\mathbf{x}_{t})$ rather than $\mathbf{x}_{t}^{*}$ in \eqref{eq:dreg-ub} allows for the possibility of multiple minimizers.
\item Second, we apply algorithms that produce a sequence $\widehat{\mathbf{x}}_{1:T} \subset \mathcal{X}$ of \emph{surrogate actions} such that the sum on the right-hand side of \eqref{eq:dreg-ub} attains a sublinear upper bound $\widehat{B}_{T}$, i.e.
\begin{equation}
\sum_{t=1}^{T}\left\Vert\widehat{\mathbf{x}}_{t} - \Pi_{\mathcal{F}_{t}^{*}}(\widehat{\mathbf{x}}_{t})\right\Vert
\leq \widehat{B}_{T}.
\end{equation}
%Note that we have used the notation $\widehat{\mathbf{x}}_{t}$ to distinguish the choices made by such \emph{dynamic uniclass algorithms}, as we would call them, from those by algorithms acting directly on the original cost-function sequence $f_{1:T}$.
This step requires users to specify a uniclass loss function $\ell: \mathcal{X} \times \mathcal{X} \to \mathbb{R}_{+}$, thus offering them the opportunity to control and optimize $\widehat{B}_{T}$.
\item Finally, by submitting $\widehat{\mathbf{x}}_{t}$ and incurring an instantaneous cost of $f_{t}(\widehat{\mathbf{x}}_{t})$ on each round $t$, the aforementioned bound $\widehat{B}_{T}$ also applies to $\text{D-Regret}_{f}$, by virtue of the Lipschitz assumption.
\end{itemize}

The beauty of our approach is that by allowing users to choose the uniclass loss function, they gain full control over the dynamic regret rate. This means that by strategically picking the uniclass loss, the user can achieve state-of-the-art regret rates in her application(s) of interest, \emph{regardless} of the type of cost functions returned by the environment. This tremendous benefit comes at the relatively small expense of having to minimize $f_{t}$ at each step, for which we can readily apply off-the-shelf convex optimization techniques \cite{boyd04, nocedal06}. Additional advantages of our universal strategy include:
\begin{itemize}
\item Unlike existing universal algorithms, it does not require running a set of OCO algorithms with different configurations in parallel, thereby avoiding computational bottlenecks.
\item The ability to handle an evolving convexity structure in the cost functions, i.e. situations in which the environment alternates, possibly in a random manner, between convex, strongly convex and exp-concave functions.
\item Its agnostic nature with regard to the modulus of strong convexity or exponential concavity where applicable.
\item In theory, it can handle non-convex cost functions, as long as these are sufficiently `well-behaved' for the problem $\min_{\mathbf{x} \in \mathcal{X}} f_{t}(\mathbf{x})$ to be computable.
\end{itemize}
An apparent drawback of our framework is that it requires the entire loss function $f_{t}$ to be revealed at each stage, thereby limiting it to full-information feedback scenarios. Nevertheless, it remains a general, original and simple approach to universal dynamic optimization that merits the attention of all concerned.

\section{Related Work}

%In this section, we briefly review related work on static regret, dynamic regret and uniclass prediction.
%
%\subsection{Static Regret}
%
%There is a rich body of literature on static regret minimization \cite{cesa-bianchi06, shalev-shwartz11, hazan22}. For general convex cost functions, the most popular algorithm is online gradient descent (OGD), which attains $O(\sqrt{T})$ regret by using a learning-rate schedule such that $\eta_{t} = O(1 / \sqrt{t})$ \cite{zinkevich03}. For $\lambda$-strongly convex functions, the regret bound can be improved to $O(\frac{1}{\lambda}\log{T})$ by running OGD with learning rates $\eta_{t} = O(1 / [\lambda t])$ \cite{shalev-shwartz07}. For $\alpha$-exp-concave functions, online Newton step achieves $O(\frac{d}{\alpha}\log{T})$ with prior knowledge of the parameter $\alpha$, where $d$ is the dimensionality \cite{hazan07a}. The aforementioned regret bounds are known to be minimax optimal for the corresponding type of functions, which means that they cannot be improved in the worst case \cite{ordentlich98, abernethy08}, but choosing the optimal algorithm for a specific problem requires domain knowledge.
%
%The study of universal algorithms for OCO originated in the adaptive online gradient descent algorithm proposed in \cite{bartlett07} and its proximal extension \cite{do09}.

In this section, we briefly review related work on dynamic regret and uniclass prediction.

\subsection{Dynamic Regret}

It is well-known that in the worst case, it is impossible to achieve sublinear dynamic regret without imposing some regularity constraints on the sequence of cost functions or comparators \cite{besbes15, jadbabaie15, yang16}. The reason is because drastic fluctuations in these sequences can render the problem intractable.
There are mainly three kinds of regularity measures that have been used in the literature. Specifically:
\begin{itemize}
\item The \emph{path length}, introduced by Zinkevich \yrcite{zinkevich03} and defined as
\begin{equation}
\label{eq:path-length}
\mathcal{P}(\mathbf{u}_{1:T})
\equiv \sum_{t=2}^{T} \Vert\mathbf{u}_{t} - \mathbf{u}_{t-1}\Vert.
\end{equation}
\item The \emph{squared path length} \cite{zhang17} which, as its name suggests, takes the square of the summands in the path length:
\begin{equation}
\label{eq:sq-path-length}
\mathcal{S}(\mathbf{u}_{1:T})
\equiv \sum_{t=2}^{T} \Vert\mathbf{u}_{t} - \mathbf{u}_{t-1}\Vert^{2}.
\end{equation}
\item The \emph{variation functional} \cite{besbes15}, defined by
\begin{equation}
\mathcal{V}_{T}^{f}
\equiv \sum_{t=2}^{T} \ \sup_{\mathbf{x} \in \mathcal{X}} \ \left|f_{t}(\mathbf{x}) - f_{t-1}(\mathbf{x})\right|.
\end{equation}
\end{itemize}
When the comparators are the possibly non-unique optimal actions, we shall use the symbols
\begin{equation}
\label{eq:generalized-path-length-measures}
\mathcal{P}_{T}^{*}
\equiv \sum_{t=2}^{T}\max_{\mathbf{x} \in \mathcal{X}}\left\Vert\Pi_{\mathcal{F}_{t}^{*}}(\mathbf{x}) - \Pi_{\mathcal{F}_{t-1}^{*}}(\mathbf{x})\right\Vert
\quad \text{and} \quad
\mathcal{S}_{T}^{*}
\equiv \sum_{t=2}^{T}\max_{\mathbf{x} \in \mathcal{X}}\left\Vert\Pi_{\mathcal{F}_{t}^{*}}(\mathbf{x}) - \Pi_{\mathcal{F}_{t-1}^{*}}(\mathbf{x})\right\Vert^{2}
\end{equation}
to represent the path-length variation, where $\mathcal{F}_{t}^{*}$ denotes the set of minimizers of $f_{t}$ and $\Pi_{\mathcal{X}}(\cdot)$ is the Euclidean projection operator (see \cref{notation}). When $\mathcal{F}_{t}^{*}$ is a singleton, i.e. when $f_{t}$ admits a unique minimizer, denoted $\mathbf{x}_{t}^{*}$, the quantities in \cref{eq:generalized-path-length-measures} boil down to $\mathcal{P}_{T}^{*} = \sum_{t=2}^{T} \Vert\mathbf{x}_{t}^{*} - \mathbf{x}_{t-1}^{*}\Vert$ and $\mathcal{S}_{T}^{*} = \sum_{t=2}^{T} \Vert\mathbf{x}_{t}^{*} - \mathbf{x}_{t-1}^{*}\Vert^{2}$, respectively.
As shown in \cite{jadbabaie15}, the aforementioned regularity measures are generally incomparable and favored in different scenarios. As discussed in \cite{zhang17}, the squared path length could be much smaller than the plain path length. For example, when $\Vert\mathbf{x}_{t}^{*} - \mathbf{x}_{t-1}^{*}\Vert = \Omega(1/\sqrt{T})$ for all $t \in [T]$, we have $\mathcal{P}_{T}^{*} = \Omega(\sqrt{T})$ but $\mathcal{S}_{T}^{*} = \Omega(1)$.

When the path length $\mathcal{P}_{T}^{*}$ is known in advance, the dynamic regret of online gradient descent is at most $O(\sqrt{T(1+\mathcal{P}_{T}^{*})})$ for general convex functions \cite{zinkevich03, yang16}. Yang et al. \yrcite{yang16} further show that an $O(\mathcal{P}_{T}^{*})$ rate is attainable for convex and smooth functions, provided the set $\mathcal{F}_{t}^{*}$ of minimizers lies in the interior of the feasible set $\mathcal{X}$.
For strongly convex and smooth functions, Mokhtari et al. \yrcite{mokhtari16} show that an $O(\mathcal{P}_{T}^{*})$ bound is achievable without knowing $\mathcal{P}_{T}^{*}$. Zhang et al. \yrcite{zhang17} propose to repeatedly apply gradient descent on each round, and demonstrate this reduces the $O(\mathcal{P}_{T}^{*})$ bound to $O(\min\{\mathcal{P}_{T}^{*},\ \mathcal{S}_{T}^{*}\})$, provided the gradients at the minimizers $\mathbf{x}_{t}^{*}$ are small. This dynamic regret guarantee has been recently enhanced to $O(\min\{\mathcal{P}_{T}^{*},\ \mathcal{S}_{T}^{*},\ \mathcal{V}_{T}^{f}\})$ via an improved analysis in \cite{zhao21}.

The aforementioned results use the path or squared path length as the regularity metric. In another line of research, researchers rely on the variation with respect to the function values. Specifically, Besbes et al. \yrcite{besbes15} show that OGD with a restarting strategy attains an $O(T^{2/3}(\mathcal{V}_{T}^{f})^{1/3})$ regret when the variation functional $\mathcal{V}_{T}^{f}$ is available ahead of time. This bound is improved to $O(T^{1/3}(\mathcal{V}_{T}^{f})^{2/3})$ for the one-dimensional squared loss in \cite{baby19}.

Although there exist abundant algorithms and theories for dynamic regret minimization, how to choose them in practice is a non-trivial task. To ensure good performance, we not only need to know the type of cost functions, but also estimate the modulus of strong convexity or exponential concavity where applicable. This reliance on human discretion limits the application of dynamic OCO techniques to real-world problems, and motivates the development of universal methods.

\subsection{Uniclass Prediction}

Uniclass prediction \cite{crammer06, crammer03} involves predicting a sequence $\mathbf{y}_{1:T}$ of vectors, where each $\mathbf{y}_{t} \in \mathbb{R}^{n}$. This task is fundamentally different from classification and regression, since predictions $\widehat{\mathbf{y}}_{t}$ must be made without observing any external input (e.g. a feature vector). Crammer et al. \yrcite{crammer06, crammer03} measure loss using the following $\epsilon$-insensitive loss function:
\begin{equation}
\label{eq:ilf}
\ell_{\epsilon}(\mathbf{y}, \widehat{\mathbf{y}})
= \begin{cases}
	0 & \text{if } \Vert\mathbf{y} - \widehat{\mathbf{y}}\Vert \leq \epsilon \\
	\Vert\mathbf{y} - \widehat{\mathbf{y}}\Vert - \epsilon & \text{otherwise}.
  \end{cases}
\end{equation}
Subsequently, they derive the uniclass predictions by solving the constrained minimization problem
\begin{equation}
\widehat{\mathbf{y}}_{t+1} = \argmin_{\widehat{\mathbf{y}} \in \mathbb{R}^{n}} \ \frac{1}{2}\Vert\hat{\mathbf{y}} - \widehat{\mathbf{y}}_{t}\Vert^{2}
\quad \text{s.t.} \quad \ell_{\epsilon}(\mathbf{y}_{t}, \widehat{\mathbf{y}}) = 0,
\end{equation}
which leads to the following passive-aggressive (PA) updates\footnote{Interestingly, \cref{eq:uniclass-updates} is a special case of the online gradient descent algorithm \cite{zinkevich03}, with a learning-rate schedule given by $\eta_{t} = \ell_{\epsilon}(\mathbf{y}_{t}, \widehat{\mathbf{y}}_{t})$ for all $t$.}:
\begin{equation}
\label{eq:uniclass-updates}
\widehat{\mathbf{y}}_{t+1}
= \widehat{\mathbf{y}}_{t} \ + \ \ell_{\epsilon}(\mathbf{y}_{t}, \widehat{\mathbf{y}}_{t})\frac{\mathbf{y}_{t} - \widehat{\mathbf{y}}_{t}}{\Vert\mathbf{y}_{t} - \widehat{\mathbf{y}}_{t}\Vert}.
\end{equation}
Lastly, Crammer et al. \yrcite{crammer06, crammer03} establish a regret bound for their uniclass PA algorithm by assuming that there exist $\widehat{\mathbf{y}}^{*}$ and $\epsilon^{*}$ such that $\ell_{\epsilon^{*}}(\mathbf{y}_{t}, \widehat{\mathbf{y}}^{*}) = 0$ for all $t$, i.e. they focus on static rather than dynamic regret. Still, their work was instrumental in the development of our universal framework, albeit we shall extend the meaning of uniclass prediction to include target sequences $\mathbf{y}_{1:T}$ within a convex subset $\mathcal{X} \subseteq \mathbb{R}^{n}$.

\section{Preliminaries}
\label{preliminaries}

\subsection{Notation}

Vectors are denoted by lower-case bold Roman letters such as $\mathbf{x}$, and all vectors are assumed to be column vectors. Uppercase bold Roman letters, such as $\mathbf{M}$, stand for matrices. As for other mathematical symbols used herein, we find it convenient to list them in \cref{notation}.

\begin{table}[t]
\caption{Mathematical notation.}
\label{notation}
\vskip 0.15in
\begin{center}
\begin{tabular*}{\linewidth}{@{\extracolsep{\fill}} ll}
\toprule
\textsc{Symbol} & \textsc{Meaning} \\
\midrule
$[n]$ & $\{1, 2, \ldots, n\}$ \\
$\mathbf{x}_{1:T}$ & the sequence $\{\mathbf{x}_{t}\}_{t=1}^{T}$ \\
$\nabla$ or $\nabla_{\mathbf{x}}$ & gradient (w.r.t. $\mathbf{x}$) \\
%$\nabla^{2}$ & the (Hessian) matrix of second derivatives \\
$\mathbf{x}^{\intercal}$ & the transpose of vector $\mathbf{x}$ \\
$\Vert\mathbf{x}\Vert$ & the Euclidean norm of vector $\mathbf{x}$ \\
%$\Vert\mathbf{x}\Vert_{\infty}$ & the maximum norm, i.e. $\max_{1 \leq i \leq n} \ |x_{i}|$ \\
$\Pi_{\mathcal{X}}(\mathbf{y})$ & the Euclidean projection, i.e. $\argmin_{\mathbf{x} \in \mathcal{X}}\Vert\mathbf{x} - \mathbf{y}\Vert$ \\
%$\mathbf{I}$ & an identity matrix of conformable dimensions \\
%$\mathbf{J}$ & the Jacobian matrix \\
%$\mathbf{A} \preceq \mathbf{B}$ & the matrix $\mathbf{B} - \mathbf{A}$ is positive semidefinite \\
\bottomrule
\end{tabular*}
\end{center}
\vskip -0.1in
\end{table}

\subsection{Definitions}

To make our paper self-contained, we now provide some definitions that will be used throughout.

\begin{definition}[\textbf{Arg min}]
\label{def:argmin}
Consider a function $f:\mathcal{X} \to \mathbb{R}$. The \emph{argument of the minimum}, or \emph{arg min} for short, of $f$ over $\mathcal{X}$ is defined by
\begin{equation}
\label{eq:argmin}
\argmin_{\mathbf{x} \in \mathcal{X}} f(\mathbf{x})
\equiv \Big\{\mathbf{x} \in \mathcal{X} : f(\mathbf{x}) \leq f(\mathbf{y}),\ \forall \mathbf{y} \in \mathcal{X}\Big\},
\end{equation}
i.e. all the points $\mathbf{x}$ in the feasible set $\mathcal{X}$ for which $f(\mathbf{x})$ attains its smallest value.
\end{definition}

\begin{definition}[\textbf{Strong convexity}]
\label{def:strong-convexity}
A function $f:\mathcal{X} \to \mathbb{R}$ is \emph{$\lambda$-strongly convex} if
\begin{equation}
\label{eq:strong-convexity1}
f(\mathbf{y}) \geq f(\mathbf{x}) + \nabla f(\mathbf{x})^{\intercal}(\mathbf{y}-\mathbf{x}) + \frac{\lambda}{2}\left\Vert\mathbf{y} - \mathbf{x}\right\Vert^{2}, \quad \forall \mathbf{x}, \mathbf{y} \in \mathcal{X}.
\end{equation}
\end{definition}

\begin{definition}[\textbf{Smoothness}]
\label{def:smoothness}
A function $f:\mathcal{X} \to \mathbb{R}$ is \emph{$L$-smooth} if
\begin{equation}
\label{eq:smoothness1}
f(\mathbf{y}) \leq f(\mathbf{x}) + \nabla f(\mathbf{x})^{\intercal}(\mathbf{y}-\mathbf{x}) + \frac{L}{2}\left\Vert\mathbf{y} - \mathbf{x}\right\Vert^{2}, \quad \forall \mathbf{x}, \mathbf{y} \in \mathcal{X}.
\end{equation}
\end{definition}

\subsection{Assumptions}

In developing our universal strategy, we assume the following assumptions about the cost-function sequence $f_{1:T}$ hold.

\begin{assumption}[\textbf{Lipschitz continuity}]
The cost functions $f_{t}:\mathcal{X} \to \mathbb{R}$ are \emph{Lipschitz continuous} over the set $\mathcal{X}$ with constant $K_{f} < \infty$, i.e.
\begin{equation}
\left|f_{t}(\mathbf{x}) - f_{t}(\mathbf{y})\right| \leq K_{f}\left\Vert\mathbf{x} - \mathbf{y}\right\Vert,
\end{equation}
for any $\mathbf{x}, \mathbf{y} \in \mathcal{X}$ and $t \in [T]$.
\label{ass:lipschitz}
\end{assumption}
The following assumption will only be required to establish a $O(\min\{\mathcal{P}_{T}^{*},\ \mathcal{S}_{T}^{*}\})$ bound on the uniclass online multiple gradient descent algorithm (see \cref{alg:omgd} and \cref{omgd}).
\begin{assumption}[\textbf{Smoothness}]
The cost functions $f_{t}:\mathcal{X} \to \mathbb{R}$ are \emph{smooth} over the set $\mathcal{X}$ with constant $L_{f}$, i.e.
\begin{equation}
f_{t}(\mathbf{y}) \leq f_{t}(\mathbf{x}) + \nabla f_{t}(\mathbf{x})^{\intercal}(\mathbf{y}-\mathbf{x}) + \frac{L_{f}}{2}\left\Vert\mathbf{y} - \mathbf{x}\right\Vert^{2},
\end{equation}
for any $\mathbf{x}, \mathbf{y} \in \mathcal{X}$ and $t \in [T]$.
\label{ass:smoothness}
\end{assumption}

\section{Our Universal Strategy}
\label{universal}

\subsection{Motivation}
\label{motivation}

We begin by noting that, for any given $t \in [T]$, Assumption \ref{ass:lipschitz} trivially implies that
\begin{equation}
\label{eq:lipschitz-corollary}
f_{t}(\mathbf{x}_{t}) - \min_{\mathbf{x} \in \mathcal{X}} f_{t}(\mathbf{x})
= f_{t}(\mathbf{x}_{t}) - f_{t}\left(\Pi_{\mathcal{F}_{t}^{*}}(\mathbf{x}_{t})\right)
\leq K_{f}\left\Vert\mathbf{x}_{t} - \Pi_{\mathcal{F}_{t}^{*}}(\mathbf{x}_{t})\right\Vert,
\end{equation}
where $\Pi_{\mathcal{X}}(\cdot)$ denotes the Euclidean projection operator (see \cref{notation}) and $\mathcal{F}_{t}^{*} \equiv \argmin_{\mathbf{x} \in \mathcal{X}} f_{t}(\mathbf{x})$ (see Definition \ref{def:argmin}).
Summing the left- and rightmost sides of \eqref{eq:lipschitz-corollary} over $t \in [T]$ yields
\begin{equation}
\label{eq:main}
\text{D-Regret}_{f}(\mathbf{x}_{1:T})
= \sum_{t=1}^{T} f_{t}(\mathbf{x}_{t}) - \sum_{t=1}^{T} \min_{\mathbf{x} \in \mathcal{X}} f_{t}(\mathbf{x})
\leq K_{f}\sum_{t=1}^{T}\left\Vert\mathbf{x}_{t} - \Pi_{\mathcal{F}_{t}^{*}}(\mathbf{x}_{t})\right\Vert.
\end{equation}
Inequality \eqref{eq:main} forms the basis for our universal approach. It essentially allows us to divert our focus from traditional algorithms bounding $\text{D-Regret}_{f}(\mathbf{x}_{1:T})$ onto algorithms generating a sequence $\widehat{\mathbf{x}}_{1:T} \subset \mathcal{X}$ of \emph{surrogate actions} that satisfy
\begin{equation}
\label{eq:surrogate-ub}
\sum_{t=1}^{T}\left\Vert\widehat{\mathbf{x}}_{t} - \Pi_{\mathcal{F}_{t}^{*}}(\widehat{\mathbf{x}}_{t})\right\Vert \leq \widehat{B}_{T},
\end{equation}
for some scalar $\widehat{B}_{T} > 0$. Following Crammer et al. \yrcite{crammer03, crammer06}, we decided to name these surrogate algorithms \emph{dynamic uniclass prediction algorithms}. Once such an algorithm has been identified, we deploy its actions $\widehat{\mathbf{x}}_{t}$ on each round $t$ in the original problem, which by virtue of Equations \eqref{eq:main} and \eqref{eq:surrogate-ub} leads to
\begin{equation}
\label{eq:main2}
\text{D-Regret}_{f}(\widehat{\mathbf{x}}_{1:T})
= \sum_{t=1}^{T} f_{t}(\widehat{\mathbf{x}}_{t}) - \sum_{t=1}^{T} \min_{\mathbf{x} \in \mathcal{X}} f_{t}(\mathbf{x})
\leq K_{f}\sum_{t=1}^{T}\left\Vert\widehat{\mathbf{x}}_{t} - \Pi_{\mathcal{F}_{t}^{*}}(\widehat{\mathbf{x}}_{t})\right\Vert
\leq K_{f}\widehat{B}_{T},
\end{equation}
meaning the surrogate bound $\widehat{B}_{T}$ naturally carries over into the dynamic regret of the original cost functions $f_{t}$. Remarkably, this property is valid \emph{regardless} of the type of cost functions: it does not matter whether these are convex, strongly convex, exp-concave, or even alternate between those types across rounds.

To find out what characteristics our dynamic uniclass algorithm should possess, we take a closer look at the sum $\sum_{t=1}^{T}\Vert\widehat{\mathbf{x}}_{t} - \Pi_{\mathcal{F}_{t}^{*}}(\widehat{\mathbf{x}}_{t})\Vert$:
\begin{align}
& \sum_{t=1}^{T}\left\Vert\widehat{\mathbf{x}}_{t} - \Pi_{\mathcal{F}_{t}^{*}}(\widehat{\mathbf{x}}_{t})\right\Vert
= \left\Vert\widehat{\mathbf{x}}_{1} - \Pi_{\mathcal{F}_{1}^{*}}(\widehat{\mathbf{x}}_{1})\right\Vert
  + \sum_{t=2}^{T}\left\Vert\widehat{\mathbf{x}}_{t} - \Pi_{\mathcal{F}_{t}^{*}}(\widehat{\mathbf{x}}_{t})\right\Vert \\ \nonumber
&= \left\Vert\widehat{\mathbf{x}}_{1} - \Pi_{\mathcal{F}_{1}^{*}}(\widehat{\mathbf{x}}_{1})\right\Vert
   + \sum_{t=2}^{T}\left\Vert\left(\widehat{\mathbf{x}}_{t} - \Pi_{\mathcal{F}_{t-1}^{*}}(\widehat{\mathbf{x}}_{t})\right) + \left(\Pi_{\mathcal{F}_{t-1}^{*}}(\widehat{\mathbf{x}}_{t}) - \Pi_{\mathcal{F}_{t}^{*}}(\widehat{\mathbf{x}}_{t})\right)\right\Vert \\
&\leq \left\Vert\widehat{\mathbf{x}}_{1} - \Pi_{\mathcal{F}_{1}^{*}}(\widehat{\mathbf{x}}_{1})\right\Vert
	  + \sum_{t=2}^{T}\left\Vert\widehat{\mathbf{x}}_{t} - \Pi_{\mathcal{F}_{t-1}^{*}}(\widehat{\mathbf{x}}_{t})\right\Vert
	  + \sum_{t=2}^{T}\left\Vert\Pi_{\mathcal{F}_{t}^{*}}(\widehat{\mathbf{x}}_{t}) - \Pi_{\mathcal{F}_{t-1}^{*}}(\widehat{\mathbf{x}}_{t})\right\Vert, \label{eq:aggerr1}
\end{align}
where \eqref{eq:aggerr1} is a direct consequence of the triangle inequality. The trick is then to somehow make the expression $\rho\sum_{t=1}^{T}\Vert\widehat{\mathbf{x}}_{t} - \Pi_{\mathcal{F}_{t}^{*}}(\widehat{\mathbf{x}}_{t})\Vert$ appear on the right-hand side of \eqref{eq:aggerr1}, where $\rho < 1$. This can be achieved provided the surrogate action sequence $\widehat{\mathbf{x}}_{1:T}$ satisfies the relation
\begin{equation}
\label{eq:contraction}
\left\Vert\widehat{\mathbf{x}}_{t} - \Pi_{\mathcal{F}_{t-1}^{*}}(\widehat{\mathbf{x}}_{t})\right\Vert
\leq \rho\left\Vert\widehat{\mathbf{x}}_{t-1} - \Pi_{\mathcal{F}_{t-1}^{*}}(\widehat{\mathbf{x}}_{t-1})\right\Vert,
\quad t = 2, 3, \ldots, T.
%\quad \text{for some } \rho \neq 1.
\end{equation}
If \eqref{eq:contraction} is true indeed, it is possible to bound \eqref{eq:aggerr1} from above as follows:
\begin{align}
\label{eq:aggerr2}
& \sum_{t=1}^{T}\left\Vert\widehat{\mathbf{x}}_{t} - \Pi_{\mathcal{F}_{t}^{*}}(\widehat{\mathbf{x}}_{t})\right\Vert \\ \nonumber
&\leq \left\Vert\widehat{\mathbf{x}}_{1} - \Pi_{\mathcal{F}_{1}^{*}}(\widehat{\mathbf{x}}_{1})\right\Vert
	  + \rho\sum_{t=2}^{T}\left\Vert\widehat{\mathbf{x}}_{t-1} - \Pi_{\mathcal{F}_{t-1}^{*}}(\widehat{\mathbf{x}}_{t-1})\right\Vert
	  + \sum_{t=2}^{T}\left\Vert\Pi_{\mathcal{F}_{t}^{*}}(\widehat{\mathbf{x}}_{t}) - \Pi_{\mathcal{F}_{t-1}^{*}}(\widehat{\mathbf{x}}_{t})\right\Vert \\ \nonumber
&\leq \left\Vert\widehat{\mathbf{x}}_{1} - \Pi_{\mathcal{F}_{1}^{*}}(\widehat{\mathbf{x}}_{1})\right\Vert
	  + \rho\sum_{t=1}^{T}\left\Vert\widehat{\mathbf{x}}_{t} - \Pi_{\mathcal{F}_{t}^{*}}(\widehat{\mathbf{x}}_{t})\right\Vert
	  + \mathcal{P}_{T}^{*},
\end{align}
where $\mathcal{P}_{T}^{*}$ has been defined in \eqref{eq:generalized-path-length-measures}.
Regrouping terms, we conclude that, if \eqref{eq:contraction} holds, the aggregate error $\sum_{t=1}^{T}\Vert\widehat{\mathbf{x}}_{t} - \Pi_{\mathcal{F}_{t}^{*}}(\widehat{\mathbf{x}}_{t})\Vert$ will be bounded above in the following manner:
\begin{equation}
\label{eq:aggerr3}
\sum_{t=1}^{T}\left\Vert\widehat{\mathbf{x}}_{t} - \Pi_{\mathcal{F}_{t}^{*}}(\widehat{\mathbf{x}}_{t})\right\Vert
\leq \frac{\mathcal{P}_{T}^{*}}{1-\rho} + \frac{\left\Vert\widehat{\mathbf{x}}_{1} - \Pi_{\mathcal{F}_{1}^{*}}(\widehat{\mathbf{x}}_{1})\right\Vert}{1-\rho}
\equiv \widehat{B}_{T}.
\end{equation}
Combining this result with \eqref{eq:main2} gives
\begin{equation}
\label{eq:regret-bound}
\text{D-Regret}_{f}(\widehat{\mathbf{x}}_{1:T})
%= \sum_{t=1}^{T} f_{t}(\widehat{\mathbf{x}}_{t}) - \sum_{t=1}^{T} \min_{\mathbf{x} \in \mathcal{X}} f_{t}(\mathbf{x})
\leq \frac{K_{f}\mathcal{P}_{T}^{*}}{1-\rho} + \frac{K_{f}\left\Vert\widehat{\mathbf{x}}_{1} - \Pi_{\mathcal{F}_{1}^{*}}(\widehat{\mathbf{x}}_{1})\right\Vert}{1-\rho}
= O(\mathcal{P}_{T}^{*}).
\end{equation}

To summarize, when we enact the surrogate action sequence $\widehat{\mathbf{x}}_{1:T} \subset \mathcal{X}$ prescribed by a dynamic uniclass algorithm obeying \eqref{eq:contraction}, the dynamic regret of the original cost functions $f_{t}$ does not grow any faster than $\mathcal{P}_{T}^{*}$, \emph{no matter} what the type of these functions is. In particular, the regret rate of $O(\mathcal{P}_{T}^{*})$ can be attained \emph{even for} general convex cost functions. This represents a major improvement over the works by, among others, Mokhtari et al. \yrcite{mokhtari16}, wherein all cost functions are required to be strongly convex and smooth for the rate $O(\mathcal{P}_{T}^{*})$ to be feasible.
The question then arises -- which uniclass algorithms and loss functions satisfy the relation in \eqref{eq:contraction}? In the next subsection, we show that online gradient descent, equipped with a sequence of strongly convex and smooth uniclass loss functions, is a candidate for this purpose.

\subsection{Online Gradient Descent}

Before presenting the online gradient descent (OGD) method \cite{zinkevich03} in the context of uniclass prediction, we discuss the construction of our time-indexed uniclass loss function, denoted $\ell_{t}$. The latter measures the loss incurred in assuming the chosen action $\widehat{\mathbf{x}}_{t}$ belongs to the set $\mathcal{F}_{t}^{*} \equiv \argmin_{\mathbf{x} \in \mathcal{X}} f_{t}(\mathbf{x})$ of minimizers of $f_{t}$ over $\mathcal{X}$. Formally, for all $t \in [T]$,
\begin{align}
\label{eq:generic-uniclass-loss}
\ell_{t} : & \ \mathcal{X} \to \mathbb{R}_{+} \\ \nonumber
& \ \mathbf{x} \mapsto \ell_{t}(\mathbf{x}) \equiv \ell\left(\mathbf{x},\ \Pi_{\mathcal{F}_{t}^{*}}(\widehat{\mathbf{x}}_{t})\right),
\end{align}
for some \textbf{user-specified} loss function $\ell : \mathcal{X} \times \mathcal{X} \to \mathbb{R}_{+}$.

The OGD algorithm for uniclass prediction starts from any $\widehat{\mathbf{x}}_{1} \in \mathcal{X}$ and performs the update
\begin{equation}
\label{eq:ogd-update}
\widehat{\mathbf{x}}_{t+1}
= \Pi_{\mathcal{X}}\left[\widehat{\mathbf{x}}_{t} - \eta \nabla \ell_{t}(\widehat{\mathbf{x}}_{t})\right],
%\qquad \ell_{t}(\widehat{\mathbf{x}}_{t}) \equiv \ell\left(\widehat{\mathbf{x}}_{t},\ \Pi_{\mathcal{F}_{t}^{*}}(\widehat{\mathbf{x}}_{t})\right)
\quad t \in [T-1],
\end{equation}
on each round $t$, after obviously computing the quantities relevant for the evaluation of the gradient $\nabla\ell_{t}(\widehat{\mathbf{x}}_{t})$, namely $\mathcal{F}_{t}^{*}$ and $\Pi_{\mathcal{F}_{t}^{*}}(\widehat{\mathbf{x}}_{t})$. The procedure is detailed in \cref{alg:ogd}.
\begin{algorithm}[t]
   \caption{Online Gradient Descent for Uniclass Prediction}
   \label{alg:ogd}
\begin{algorithmic}[1]
   \STATE {\bfseries Input:}
   \begin{itemize}
   \item closed and convex set $\mathcal{X} \subseteq \mathbb{R}^{n}$
   \item uniclass loss function $\ell$ as per \cref{eq:generic-uniclass-loss}
   \item learning rate $\eta > 0$
   \end{itemize}
   \STATE {\bfseries Initialize:} Pick $\widehat{\mathbf{x}}_{1}$ arbitrarily in $\mathcal{X}$
   \FOR{$t=1$ {\bfseries to} $T-1$}
   \STATE Play $\widehat{\mathbf{x}}_{t}$
   \STATE Observe the cost function $f_{t}:\mathcal{X} \to \mathbb{R}$ and compute the quantities
   \begin{equation*}
   \mathcal{F}_{t}^{*} = \argmin_{\mathbf{x} \in \mathcal{X}} \ f_{t}(\mathbf{x}), \quad
   \Pi_{\mathcal{F}_{t}^{*}}(\widehat{\mathbf{x}}_{t}) = \argmin_{\mathbf{x} \in \mathcal{F}_{t}^{*}} \ \left\Vert\mathbf{x} - \widehat{\mathbf{x}}_{t}\right\Vert
   \end{equation*}
   \STATE Update $\widehat{\mathbf{x}}_{t}$ as follows:
   \begin{equation*}
   \widehat{\mathbf{x}}_{t+1}
   = \Pi_{\mathcal{X}}\left[\widehat{\mathbf{x}}_{t} - \eta \nabla \ell_{t}(\widehat{\mathbf{x}}_{t})\right]
%   \qquad \ell_{t}(\widehat{\mathbf{x}}_{t}) \equiv \ell\left(\widehat{\mathbf{x}}_{t},\ \Pi_{\mathcal{F}_{t}^{*}}(\widehat{\mathbf{x}}_{t})\right)
   \end{equation*}
   \ENDFOR
\end{algorithmic}
\end{algorithm}
Our main result on the regret bound of \cref{alg:ogd} is derived from the following lemma that bounds the distance $\Vert\widehat{\mathbf{x}}_{t+1} - \Pi_{\mathcal{F}_{t}^{*}}(\widehat{\mathbf{x}}_{t+1})\Vert$ in terms of $\Vert\widehat{\mathbf{x}}_{t} - \Pi_{\mathcal{F}_{t}^{*}}(\widehat{\mathbf{x}}_{t})\Vert$.
\begin{lemma}
\label{main-result}
Consider the online gradient descent method for uniclass prediction, outlined in \cref{alg:ogd}. Suppose the following conditions hold for any $t \in [T]$:
\begin{enumerate}
\item the uniclass loss functions $\ell_{t}:\mathcal{X} \to \mathbb{R}_{+}$ are $\lambda_{\ell}$-strongly convex (see Definition \ref{def:strong-convexity}) and
\begin{equation}
\mathcal{L}_{t}^{*}
\equiv \argmin_{\mathbf{x} \in \mathcal{X}} \ \ell_{t}(\mathbf{x})
= \Big\{\Pi_{\mathcal{F}_{t}^{*}}(\widehat{\mathbf{x}}_{t})\Big\}.
\end{equation}
\item the uniclass loss functions $\ell_{t}:\mathcal{X} \to \mathbb{R}_{+}$ are $L_{\ell}$-smooth (see Definition \ref{def:smoothness});
\item the learning rate $\eta$ is chosen such that $\eta \leq 1/L_{\ell}$.
\end{enumerate}
Then, the sequence of actions $\widehat{\mathbf{x}}_{1:T}$ generated by \cref{alg:ogd} satisfies the relation described in \eqref{eq:contraction}, i.e.
\begin{equation}
\label{eq:contraction-ogd}
\left\Vert\widehat{\mathbf{x}}_{t+1} - \Pi_{\mathcal{F}_{t}^{*}}(\widehat{\mathbf{x}}_{t+1})\right\Vert
\leq \rho\left\Vert\widehat{\mathbf{x}}_{t} - \Pi_{\mathcal{F}_{t}^{*}}(\widehat{\mathbf{x}}_{t})\right\Vert,
\quad t \in [T-1],
\end{equation}
with $\rho = \sqrt{1 - \frac{\lambda_{\ell}}{\lambda_{\ell} \ + \ \eta^{-1}}}$.
\end{lemma}
\begin{proof}
See \cref{main-result-proof}.
\end{proof}
%\vspace{-20pt}
\begin{theorem}
\label{main-result-corollary}
Under the conditions laid out in Lemma \ref{main-result}, along with Assumption \ref{ass:lipschitz}, the dynamic regret for the sequence $\widehat{\mathbf{x}}_{1:T}$ of actions generated by \cref{alg:ogd} is bounded above as follows:
\begin{equation}
\label{eq:regret-bound2}
\sum_{t=1}^{T} f_{t}(\widehat{\mathbf{x}}_{t}) - \sum_{t=1}^{T} \min_{\mathbf{x} \in \mathcal{X}} f_{t}(\mathbf{x})
\leq \frac{K_{f}\mathcal{P}_{T}^{*}}{1-\rho} + \frac{K_{f}\left\Vert\widehat{\mathbf{x}}_{1} - \widehat{\mathbf{x}}_{1}^{*}\right\Vert}{1-\rho},
\end{equation}
where $\rho = \sqrt{1 - \frac{\lambda_{\ell}}{\lambda_{\ell} \ + \ \eta^{-1}}}$ and $\widehat{\mathbf{x}}_{1}^{*} \equiv \Pi_{\mathcal{F}_{1}^{*}}(\widehat{\mathbf{x}}_{1})$.
\end{theorem}
\begin{proof}
See \cref{motivation}, especially Equations \eqref{eq:aggerr2}--\eqref{eq:regret-bound}.
\end{proof}
It is crucial to note that \textbf{the regret bound in \eqref{eq:regret-bound2} is achieved by any arbitrary sequence $f_{1:T}$ of cost functions}. In particular, we do not require any additional desirable properties such as strong convexity, smoothness or exp-concavity, unlike previous studies. Mokhtari et al. \yrcite{mokhtari16} derive the same $O(\mathcal{P}_{T}^{*})$ bound for exponentially-weighted OGD updates, but their results rely on the assumptions that the cost functions $f_{t}$ are strongly convex, smooth and have bounded gradient norms. Additionally, we remark that \cref{main-result-corollary} assumes neither a prior knowledge nor an online feedback about $\mathcal{P}_{T}^{*}$.

At this point, a natural question that arises is the following: what kind of uniclass loss functions satisfy the assumptions made by Lemma \ref{main-result}? We provide a compelling example below.

\subsubsection{An Example of Uniclass Loss Function}

To illustrate \cref{alg:ogd}, we discuss an example of uniclass loss function that abides by the assumptions of Lemma \ref{main-result}, namely the squared loss defined by
\begin{equation}
\label{eq:sqloss}
\ell_{t}^{\text{sq}}(\mathbf{x})
\equiv \frac{1}{2}\left\Vert\mathbf{x} - \Pi_{\mathcal{F}_{t}^{*}}(\widehat{\mathbf{x}}_{t})\right\Vert^{2}.
\end{equation}
This is perhaps the most natural uniclass loss function one can think of. It is trivial to show that it is $\lambda$-strongly convex and $L$-smooth, for $\lambda \leq 1$ and $L \geq 1$. Furthermore, by virtue of the properties of the norm,
\begin{equation}
\label{eq:sqloss-min}
\mathcal{L}_{t}^{*}
= \argmin_{\mathbf{x} \in \mathcal{X}} \ \ell_{t}^{\text{sq}}(\mathbf{x})
= \Big\{\Pi_{\mathcal{F}_{t}^{*}}(\widehat{\mathbf{x}}_{t})\Big\}.
\end{equation}

\subsection{Online Multiple Gradient Descent}

A convenient feature of our universal strategy is that the user can access the gradient of each loss function $\ell_{t}$ more than once, since they get to specify these functions in the first place, which naturally places them in a full-information feedback scenario.
Now, when a learner is allowed to query the loss gradient multiple times, it is possible to achieve tighter dynamic regret bounds, as demonstrated in \cite{zhang17, zhao21}. %The intuition behind this is that by repeatedly applying gradient descent, the learner is able to extract more information from each function.
For this reason, Zhang et al. \yrcite{zhang17} devised a variant of online gradient descent, called \emph{online multiple gradient descent} (OMGD), wherein the loss gradient is queried multiple times on each round. Specifically, on round $t$, given the current action $\widehat{\mathbf{x}}_{t}$, OMGD generates a sequence $\widehat{\mathbf{z}}_{t}^{(1)}, \widehat{\mathbf{z}}_{t}^{(2)}, \ldots, \widehat{\mathbf{z}}_{t}^{(m)}, \widehat{\mathbf{z}}_{t}^{(m+1)}$, where $m$ represents the total number of inner iterations, a constant \emph{independent} of the time horizon $T$. The starting point is $\widehat{\mathbf{z}}_{t}^{(1)} = \widehat{\mathbf{x}}_{t}$, followed by the update rule
\begin{equation}
\label{eq:omgd-update}
\widehat{\mathbf{z}}_{t}^{(i+1)}
= \Pi_{\mathcal{X}}[\widehat{\mathbf{z}}_{t}^{(i)} - \eta\nabla\ell_{t}(\widehat{\mathbf{z}}_{t}^{(i)})],
\quad i \in [m],
\end{equation}
where $i$ denotes the index of the inner loop. The subsequent action $\widehat{\mathbf{x}}_{t+1}$ is chosen to be the output of the inner loop, i.e. $\widehat{\mathbf{x}}_{t+1} = \widehat{\mathbf{z}}_{t}^{(m+1)}$. The entire procedure is summarized in \cref{alg:omgd}.
\begin{algorithm}[t]
   \caption{Online Multiple Gradient Descent for Uniclass Prediction}
   \label{alg:omgd}
\begin{algorithmic}[1]
   \STATE {\bfseries Input:}
   \begin{itemize}
   \item closed and convex set $\mathcal{X} \subseteq \mathbb{R}^{n}$
   \item uniclass loss function $\ell$ as per \cref{eq:generic-uniclass-loss}
   \item learning rate $\eta > 0$
   \item number of inner iterations $m$
   \end{itemize}
   \STATE {\bfseries Initialize:} Pick $\widehat{\mathbf{x}}_{1}$ arbitrarily in $\mathcal{X}$
   \FOR{$t=1$ {\bfseries to} $T-1$}
   \STATE Submit $\widehat{\mathbf{x}}_{t}$
   \STATE Observe the cost function $f_{t}:\mathcal{X} \to \mathbb{R}$ and compute the quantities
   \begin{equation*}
   \mathcal{F}_{t}^{*} = \argmin_{\mathbf{x} \in \mathcal{X}} \ f_{t}(\mathbf{x}), \quad
   \Pi_{\mathcal{F}_{t}^{*}}(\widehat{\mathbf{x}}_{t}) = \argmin_{\mathbf{x} \in \mathcal{F}_{t}^{*}} \ \left\Vert\mathbf{x} - \widehat{\mathbf{x}}_{t}\right\Vert
   \end{equation*}
   \STATE Set $\widehat{\mathbf{z}}_{t}^{(1)} = \widehat{\mathbf{x}}_{t}$
   \FOR{$i=1$ {\bfseries to} $m$}
   \STATE
   \begin{equation*}
	\widehat{\mathbf{z}}_{t}^{(i+1)}
	= \Pi_{\mathcal{X}}[\widehat{\mathbf{z}}_{t}^{(i)} - \eta\nabla\ell_{t}(\widehat{\mathbf{z}}_{t}^{(i)})]
   \end{equation*}
   \ENDFOR
   \STATE {\bfseries Output:} $\widehat{\mathbf{x}}_{t+1} = \widehat{\mathbf{z}}_{t}^{(m+1)}$
   \ENDFOR
\end{algorithmic}
\end{algorithm}

By applying gradient descent multiple times, we are able to extract more information from each loss function, and are therefore more likely to obtain a tighter dynamic regret bound. The following theorem shows that repeated access to the gradient indeed helps improve dynamic regret.
\begin{theorem}
\label{omgd}
Suppose the assumptions regarding the uniclass loss functions $\ell_{t}$ in Lemma \ref{main-result}, along with Assumptions \ref{ass:lipschitz} and \ref{ass:smoothness} about the original cost functions $f_{t}$, are true. By setting $\eta \leq 1/L_{\ell}$ and $m = \lceil\frac{\lambda_{\ell} \ + \ \eta^{-1}}{\lambda_{\ell}} \mathrm{ln}(4)\rceil$ in \cref{alg:omgd}, we have for any $\alpha > 0$
\begin{align}
\label{eq:omgd-regret-bound}
& \sum_{t=1}^{T} f_{t}(\widehat{\mathbf{x}}_{t}) - \sum_{t=1}^{T} \min_{\mathbf{x} \in \mathcal{X}} f_{t}(\mathbf{x}) \\ \nonumber
&\leq \min\bigg\{2K_{f}\mathcal{P}_{T}^{*} + 2K_{f}\left\Vert\widehat{\mathbf{x}}_{1} - \widehat{\mathbf{x}}_{1}^{*}\right\Vert,\; \frac{G_{T}^{*}}{2\alpha} + 2(L_{f}+\alpha)\mathcal{S}_{T}^{*} + (L_{f}+\alpha)\left\Vert\widehat{\mathbf{x}}_{1} - \widehat{\mathbf{x}}_{1}^{*}\right\Vert^{2}\bigg\},
\end{align}
where $\widehat{\mathbf{x}}_{1}^{*} \equiv \Pi_{\mathcal{F}_{1}^{*}}(\widehat{\mathbf{x}}_{1})$, $\mathcal{S}_{T}^{*}$ is defined in \cref{eq:generalized-path-length-measures} and
\begin{equation}
{G}_{T}^{*}
\equiv \max_{\{\widehat{\mathbf{x}}_{t}^{*} \in \mathcal{F}_{t}^{*}\}_{t=1}^{T}} \ \sum_{t=1}^{T}\left\Vert\nabla f_{t}(\widehat{\mathbf{x}}_{t}^{*})\right\Vert^{2}.
\end{equation}

Furthermore, suppose $G_{T}^{*} = O(\mathcal{S}_{T}^{*})$. In this case, \cref{eq:omgd-regret-bound} implies
\begin{equation}
\sum_{t=1}^{T} f_{t}(\widehat{\mathbf{x}}_{t}) - \sum_{t=1}^{T} \min_{\mathbf{x} \in \mathcal{X}} f_{t}(\mathbf{x})
= O(\min\{\mathcal{P}_{T}^{*},\ \mathcal{S}_{T}^{*}\}).
\end{equation}
\end{theorem}
Compared to \cref{main-result-corollary}, the OMGD method reduces the dynamic regret rate from $O(\mathcal{P}_{T}^{*})$ to $O(\min\{\mathcal{P}_{T}^{*},\ \mathcal{S}_{T}^{*}\})$, provided the cost-function gradients evaluated at the minimizers are small.

\section{Conclusion and Future Work}

In this paper, we propose an original and intelligible strategy for universal online optimization in changing environments that can handle convex, strongly convex and exp-concave cost functions simultaneously. To deal with the uncertainty of online cost functions, we reduce the original problem to a uniclass prediction problem to which we apply existing dynamic OCO algorithms, after strategically selecting a uniclass loss function. Assuming the original cost functions are Lipschitz continuous, the thus derived dynamic regret rate naturally carries over into the original problem.

There are certainly several avenues for future research. First, our universal framework is designed for the purpose of minimizing the worst-case dynamic regret. Going forward, we will investigate how to modify it to support stronger notions of regret, such as general dynamic regret \cite{zinkevich03} and adaptive regret \cite{hazan07b, daniely15}. On the other hand, as we mentioned in the introduction, our method should be able to handle non-convex cost functions as well, at least in theory. Since this would constitute a major breakthrough, we shall devote considerable time and effort to study whether and how it can be achieved.

\bibliography{colt2023}

\begin{thebibliography}{28}
\providecommand{\natexlab}[1]{#1}
\providecommand{\url}[1]{\texttt{#1}}
\expandafter\ifx\csname urlstyle\endcsname\relax
  \providecommand{\doi}[1]{doi: #1}\else
  \providecommand{\doi}{doi: \begingroup \urlstyle{rm}\Url}\fi

\bibitem[Baby and Wang(2019)]{baby19}
D.~Baby and Y.-X. Wang.
\newblock {Online Forecasting of Total-Variation-bounded Sequences}.
\newblock In H.~Wallach, H.~Larochelle, A.~Beygelzimer, F.~d\textquotesingle
  Alch\'{e}-Buc, E.~Fox, and R.~Garnett, editors, \emph{Advances in Neural
  Information Processing Systems}, volume~32. Curran Associates, Inc., 2019.

\bibitem[Bartlett et~al.(2007)Bartlett, Hazan, and Rakhlin]{bartlett07}
P.~Bartlett, E.~Hazan, and A.~Rakhlin.
\newblock {Adaptive Online Gradient Descent}.
\newblock In J.~Platt, D.~Koller, Y.~Singer, and S.~Roweis, editors,
  \emph{Advances in Neural Information Processing Systems}, volume~20. Curran
  Associates, Inc., 2007.

\bibitem[Besbes et~al.(2015)Besbes, Gur, and Zeevi]{besbes15}
O.~Besbes, Y.~Gur, and A.~Zeevi.
\newblock {Non-Stationary Stochastic Optimization}.
\newblock \emph{Operations Research}, 63\penalty0 (5):\penalty0 1227--1244,
  September 2015.

\bibitem[Boyd and Vandenberghe(2004)]{boyd04}
S.~Boyd and L.~Vandenberghe.
\newblock \emph{Convex Optimization}.
\newblock Cambridge University Press, 2004.

\bibitem[Cesa-Bianchi and Lugosi(2006)]{cesa-bianchi06}
N.~Cesa-Bianchi and G.~Lugosi.
\newblock \emph{Prediction, Learning and Games}.
\newblock Cambridge University Press, 2006.

\bibitem[Crammer et~al.(2003)Crammer, Dekel, Shalev-Shwartz, and
  Singer]{crammer03}
K.~Crammer, O.~Dekel, S.~Shalev-Shwartz, and Y.~Singer.
\newblock {Online Passive-Aggressive Algorithms}.
\newblock In S.~Thrun, L.~Saul, and B.~Sch\"{o}lkopf, editors, \emph{Advances
  in Neural Information Processing Systems}, volume~16. MIT Press, 2003.

\bibitem[Crammer et~al.(2006)Crammer, Dekel, Keshet, Shalev-Shwartz, and
  Singer]{crammer06}
K.~Crammer, O.~Dekel, J.~Keshet, S.~Shalev-Shwartz, and Y.~Singer.
\newblock {Online Passive-Aggressive Algorithms}.
\newblock \emph{Journal of Machine Learning Research}, 7\penalty0
  (19):\penalty0 551--585, 2006.

\bibitem[Daniely et~al.(2015)Daniely, Gonen, and Shalev-Shwartz]{daniely15}
A.~Daniely, A.~Gonen, and S.~Shalev-Shwartz.
\newblock {Strongly Adaptive Online Learning}.
\newblock In F.~Bach and D.~Blei, editors, \emph{Proceedings of the 32nd
  International Conference on Machine Learning}, volume~37, pages 1405--1411,
  Lille, France, July 2015. PMLR.

\bibitem[Hall and Willett(2013)]{hall13}
E.~C. Hall and R.~M. Willett.
\newblock {Dynamical Models and Tracking Regret in Online Convex Programming}.
\newblock In S.~Dasgupta and D.~McAllester, editors, \emph{Proceedings of the
  30th International Conference on Machine Learning}, volume~28, pages
  579--587, Atlanta, Georgia, USA, June 2013. PMLR.

\bibitem[Hazan(2022)]{hazan22}
E.~Hazan.
\newblock \emph{Introduction to Online Convex Optimization}.
\newblock MIT Press, 2nd edition, 2022.

\bibitem[Hazan and Kale(2011)]{hazan11}
E.~Hazan and S.~Kale.
\newblock Beyond the regret minimization barrier: an optimal algorithm for
  stochastic strongly-convex optimization.
\newblock In S.~M. Kakade and U.~von Luxburg, editors, \emph{Proceedings of the
  24th Annual Conference on Learning Theory}, volume~19, pages 421--436,
  Budapest, Hungary, June 2011. PMLR.

\bibitem[Hazan and Seshadhri(2007b)]{hazan07b}
E.~Hazan and C.~Seshadhri.
\newblock {Adaptive Algorithms for Online Decision Problems}.
\newblock In \emph{Electronic Colloquium on Computational Complexity},
  volume~88, 2007b.

\bibitem[Jadbabaie et~al.(2015)Jadbabaie, Rakhlin, Shahrampour, and
  Sridharan]{jadbabaie15}
A.~Jadbabaie, A.~Rakhlin, S.~Shahrampour, and K.~Sridharan.
\newblock {Online Optimization: Competing with Dynamic Comparators}.
\newblock In G.~Lebanon and S.~V.~N. Vishwanathan, editors, \emph{Proceedings
  of the 18th International Conference on Artificial Intelligence and
  Statistics}, volume~38, pages 398--406, San Diego, California, USA, May 2015.
  PMLR.

\bibitem[Mokhtari et~al.(2016)Mokhtari, Shahrampour, Jadbabaie, and
  Ribeiro]{mokhtari16}
A.~Mokhtari, S.~Shahrampour, A.~Jadbabaie, and A.~Ribeiro.
\newblock {Online Optimization in Dynamic Environments: Improved Regret Rates
  for Strongly Convex Problems}.
\newblock In \emph{Proceedings of the 55th IEEE Conference on Decision and
  Control}, pages 7195--7201, ARIA Resort \& Casino, Las Vegas, Nevada, USA,
  December 2016.

\bibitem[Nocedal and Wright(2006)]{nocedal06}
J.~Nocedal and S.~J. Wright.
\newblock \emph{Numerical Optimization}.
\newblock Springer Series in Operations Research and Financial Engineering.
  Springer, New York, New York, USA, 2nd edition, 2006.

\bibitem[Shalev-Shwartz(2011)]{shalev-shwartz11}
S.~Shalev-Shwartz.
\newblock {Online Learning and Online Convex Optimization}.
\newblock In \emph{Foundations and Trends in Machine Learning}, volume~4, pages
  107--194. now Publishers Inc., 2011.

\bibitem[van Erven and Koolen(2016)]{vanerven16}
T.~van Erven and W.~M. Koolen.
\newblock {MetaGrad: Multiple Learning Rates in Online Learning}.
\newblock In D.~Lee, M.~Sugiyama, U.~Luxburg, I.~Guyon, and R.~Garnett,
  editors, \emph{Advances in Neural Information Processing Systems}, volume~29.
  Curran Associates, Inc., 2016.

\bibitem[Wang et~al.(2019)Wang, Lu, and Zhang]{wang19}
G.~Wang, S.~Lu, and L.~Zhang.
\newblock {Adaptivity and Optimality: A Universal Algorithm for Online Convex
  Optimization}.
\newblock In R.~P. Adams and V.~Gogate, editors, \emph{Proceedings of The 35th
  Uncertainty in Artificial Intelligence Conference}, volume 115, pages
  659--668, Tel Aviv, Israel, July 2019. PMLR.

\bibitem[Wang et~al.(2020)Wang, Lu, Hu, and Zhang]{wang20}
G.~Wang, S.~Lu, Y.~Hu, and L.~Zhang.
\newblock {Adapting to Smoothness: A More Universal Algorithm for Online Convex
  Optimization}.
\newblock In \emph{Proceedings of the 34th AAAI Conference on Artificial
  Intelligence}, pages 6162--6169, New York, New York, USA, February 2020.

\bibitem[Yang et~al.(2016)Yang, Zhang, Jin, and Yi]{yang16}
T.~Yang, L.~Zhang, R.~Jin, and J.~Yi.
\newblock {Tracking Slowly Moving Clairvoyant: Optimal Dynamic Regret of Online
  Learning with True and Noisy Gradient}.
\newblock In M.~F. Balcan and K.~Q. Weinberger, editors, \emph{Proceedings of
  the 33rd International Conference on Machine Learning}, volume~48, pages
  449--457, New York, New York, USA, June 2016. PMLR.

\bibitem[Zhang et~al.(2017)Zhang, Yang, Yi, Jin, and Zhou]{zhang17}
L.~Zhang, T.~Yang, J.~Yi, R.~Jin, and Z.-H. Zhou.
\newblock {Improved Dynamic Regret for Non-degenerate Functions}.
\newblock In I.~Guyon, U.~Von Luxburg, S.~Bengio, H.~Wallach, R.~Fergus,
  S.~Vishwanathan, and R.~Garnett, editors, \emph{Advances in Neural
  Information Processing Systems}, volume~30, 2017.

\bibitem[Zhang et~al.(2018)Zhang, Yang, Jin, and Zhou]{zhang18b}
L.~Zhang, T.~Yang, R.~Jin, and Z.-H. Zhou.
\newblock {Dynamic Regret of Strongly Adaptive Methods}.
\newblock In J.~Dy and A.~Krause, editors, \emph{Proceedings of the 35th
  International Conference on Machine Learning}, volume~80, pages 5882--5891,
  Stockholm, Sweden, July 2018. PMLR.

\bibitem[Zhang et~al.(2021)Zhang, Wang, Tu, Jiang, and Zhou]{zhang21}
L.~Zhang, G.~Wang, W.-W. Tu, W.~Jiang, and Z.-H. Zhou.
\newblock {Dual Adaptivity: A Universal Algorithm for Minimizing the Adaptive
  Regret of Convex Functions}.
\newblock In M.~Ranzato, A.~Beygelzimer, Y.~Dauphin, P.S. Liang, and J.~Wortman
  Vaughan, editors, \emph{Advances in Neural Information Processing Systems},
  volume~34, pages 24968--24980. Curran Associates, Inc., 2021.

\bibitem[Zhang et~al.(2022)Zhang, Wang, Yi, and Yang]{zhang22}
L.~Zhang, G.~Wang, J.~Yi, and T.~Yang.
\newblock {A Simple yet Universal Strategy for Online Convex Optimization}.
\newblock In K.~Chaudhuri, S.~Jegelka, L.~Song, C.~Szepesvari, G.~Niu, and
  S.~Sabato, editors, \emph{Proceedings of the 39th International Conference on
  Machine Learning}, volume 162, pages 26605--26623, Baltimore, Maryland, USA,
  July 2022.

\bibitem[Zhang et~al.(2020)Zhang, Zhao, and Zhou]{zhang20}
Y.-J. Zhang, P.~Zhao, and Z.-H. Zhou.
\newblock {A Simple Online Algorithm for Competing with Dynamic Comparators}.
\newblock In J.~Peters and D.~Sontag, editors, \emph{Proceedings of the 36th
  Conference on Uncertainty in Artificial Intelligence}, volume 124, pages
  390--399. PMLR, August 2020.

\bibitem[Zhao and Zhang(2021)]{zhao21}
P.~Zhao and L.~Zhang.
\newblock {Improved Analysis for Dynamic Regret of Strongly Convex and Smooth
  Functions}.
\newblock In A.~Jadbabaie, J.~Lygeros, G.~J. Pappas, P.~A. Parrilo, B.~Recht,
  C.~J. Tomlin, and M.~N. Zeilinger, editors, \emph{Proceedings of the 3rd
  Conference on Learning for Dynamics and Control}, volume 144, pages 48--59.
  PMLR, June 2021.

\bibitem[Zhou(2022)]{zhou22}
Z.-H. Zhou.
\newblock Open-environment machine learning.
\newblock \emph{National Science Review}, 9\penalty0 (8), 2022.

\bibitem[Zinkevich(2003)]{zinkevich03}
M.~Zinkevich.
\newblock {Online Convex Programming and Generalized Infinitesimal Gradient
  Ascent}.
\newblock In T.~Fawcett and N.~Mishra, editors, \emph{Proceedings of the 20th
  International Conference on Machine Learning}, pages 928--936, Washington,
  District of Columbia, USA, August 2003. AAAI Press.

\end{thebibliography}

\appendix

\section{Proofs from \cref{universal} -- Main Results}

\subsection{Proof of Lemma \ref{main-result}}
\label{main-result-proof}

We first introduce the following property of strongly convex functions \cite{hazan11}.
\begin{lemma}
\label{hazan}
Assume that $f:\mathcal{X} \to \mathbb{R}$ is $\lambda$-strongly convex, and let $\mathbf{x}^{*}
= \argmin_{\mathbf{x} \in \mathcal{X}} f(\mathbf{x})$. Then, we have
\begin{equation}
\label{eq:hazan}
f(\mathbf{x}) - f(\mathbf{x}^{*})
\geq \frac{\lambda}{2}\left\Vert\mathbf{x} - \mathbf{x}^{*}\right\Vert^{2},
\quad \forall \mathbf{x} \in \mathcal{X}.
\end{equation}
\end{lemma}
Additionally, the following lemma, the proof of which can be found in \cite{mokhtari16}, will be useful.
\begin{lemma}
\label{mokhtari}
Consider the update in \cref{eq:ogd-update}. Given the iterate $\widehat{\mathbf{x}}_{t}$, the instantaneous gradient $\nabla\ell_{t}(\widehat{\mathbf{x}}_{t})$ and the positive constant $\eta$, the optimal argument of the optimization problem
\begin{equation}
\label{eq:mokhtari}
\widetilde{\mathbf{x}}_{t+1}
= \argmin_{\mathbf{x} \in \mathcal{X}} \ \Big\{\nabla\ell_{t}(\widehat{\mathbf{x}}_{t})^{\intercal}(\mathbf{x}-\widehat{\mathbf{x}}_{t}) + \frac{1}{2\eta}\left\Vert\mathbf{x} - \widehat{\mathbf{x}}_{t}\right\Vert^{2}\Big\}
\end{equation}
is equal to the iterate $\widehat{\mathbf{x}}_{t+1}$ generated by \eqref{eq:ogd-update}.
\end{lemma}
It is straightforward to verify that the objective function in \eqref{eq:mokhtari} is strongly convex, with constant $\lambda = 1 / \eta$. Thus, by applying Lemma \ref{hazan} and rearranging
\begin{align}
& \ell_{t}(\widehat{\mathbf{x}}_{t}) + \nabla\ell_{t}(\widehat{\mathbf{x}}_{t})^{\intercal}(\widehat{\mathbf{x}}_{t+1} - \widehat{\mathbf{x}}_{t}) + \frac{1}{2\eta}\left\Vert\widehat{\mathbf{x}}_{t+1} - \widehat{\mathbf{x}}_{t}\right\Vert^{2} \\ \nonumber
& \leq \ell_{t}(\widehat{\mathbf{x}}_{t}) + \nabla\ell_{t}(\widehat{\mathbf{x}}_{t})^{\intercal}(\mathbf{x}-\widehat{\mathbf{x}}_{t}) + \frac{1}{2\eta}\left\Vert\mathbf{x} - \widehat{\mathbf{x}}_{t}\right\Vert^{2} - \frac{1}{2\eta}\left\Vert\mathbf{x} - \widehat{\mathbf{x}}_{t+1}\right\Vert^{2},
\end{align}
for any $\mathbf{x} \in \mathcal{X}$.
In particular,
\begin{align}
\label{eq:main-result-2}
& \ell_{t}(\widehat{\mathbf{x}}_{t}) + \nabla\ell_{t}(\widehat{\mathbf{x}}_{t})^{\intercal}(\widehat{\mathbf{x}}_{t+1} - \widehat{\mathbf{x}}_{t}) + \frac{1}{2\eta}\left\Vert\widehat{\mathbf{x}}_{t+1} - \widehat{\mathbf{x}}_{t}\right\Vert^{2} \\ \nonumber
& \leq \ell_{t}(\widehat{\mathbf{x}}_{t}) + \nabla\ell_{t}(\widehat{\mathbf{x}}_{t})^{\intercal}\left(\Pi_{\mathcal{F}_{t}^{*}}(\widehat{\mathbf{x}}_{t}) - \widehat{\mathbf{x}}_{t}\right) + \frac{1}{2\eta}\left\Vert\Pi_{\mathcal{F}_{t}^{*}}(\widehat{\mathbf{x}}_{t}) - \widehat{\mathbf{x}}_{t}\right\Vert^{2} - \frac{1}{2\eta}\left\Vert\widehat{\mathbf{x}}_{t+1} - \Pi_{\mathcal{F}_{t}^{*}}(\widehat{\mathbf{x}}_{t})\right\Vert^{2}.
\end{align}
On the other hand, from the convexity of $\ell_{t}$, it follows that
\begin{equation}
\label{eq:main-result-3}
\ell_{t}(\widehat{\mathbf{x}}_{t}) + \nabla\ell_{t}(\widehat{\mathbf{x}}_{t})^{\intercal}\left(\Pi_{\mathcal{F}_{t}^{*}}(\widehat{\mathbf{x}}_{t}) - \widehat{\mathbf{x}}_{t}\right)
\leq \ell_{t}\left(\Pi_{\mathcal{F}_{t}^{*}}(\widehat{\mathbf{x}}_{t})\right).
\end{equation}
Furthermore, the smoothness assumption (see Definition \ref{def:smoothness}), along with the choice $\eta \leq 1/L_{\ell}$, imply
\begin{align}
\label{eq:main-result-4}
\ell_{t}(\widehat{\mathbf{x}}_{t+1})
&\leq \ell_{t}(\widehat{\mathbf{x}}_{t}) + \nabla\ell_{t}(\widehat{\mathbf{x}}_{t})^{\intercal}(\widehat{\mathbf{x}}_{t+1} - \widehat{\mathbf{x}}_{t}) + \frac{L_{\ell}}{2}\left\Vert\widehat{\mathbf{x}}_{t+1} - \widehat{\mathbf{x}}_{t}\right\Vert^{2} \\ \nonumber
&\leq \ell_{t}(\widehat{\mathbf{x}}_{t}) + \nabla\ell_{t}(\widehat{\mathbf{x}}_{t})^{\intercal}(\widehat{\mathbf{x}}_{t+1} - \widehat{\mathbf{x}}_{t}) + \frac{1}{2\eta}\left\Vert\widehat{\mathbf{x}}_{t+1} - \widehat{\mathbf{x}}_{t}\right\Vert^{2}.
\end{align}
Combining Equations \eqref{eq:main-result-2}-\eqref{eq:main-result-4}, we obtain
\begin{equation}
\label{eq:main-result-5}
\ell_{t}(\widehat{\mathbf{x}}_{t+1})
\leq \ell_{t}\left(\Pi_{\mathcal{F}_{t}^{*}}(\widehat{\mathbf{x}}_{t})\right) + \frac{1}{2\eta}\left\Vert\Pi_{\mathcal{F}_{t}^{*}}(\widehat{\mathbf{x}}_{t}) - \widehat{\mathbf{x}}_{t}\right\Vert^{2} - \frac{1}{2\eta}\left\Vert\widehat{\mathbf{x}}_{t+1} - \Pi_{\mathcal{F}_{t}^{*}}(\widehat{\mathbf{x}}_{t})\right\Vert^{2}.
\end{equation}
Finally, since $\ell_{t}$ is assumed to be $\lambda_{\ell}$-strongly convex, Lemma \ref{hazan} implies
\begin{equation}
\label{eq:main-result-6}
\ell_{t}(\widehat{\mathbf{x}}_{t+1}) - \ell_{t}\left(\Pi_{\mathcal{F}_{t}^{*}}(\widehat{\mathbf{x}}_{t})\right)
\geq \frac{\lambda_{\ell}}{2}\left\Vert\widehat{\mathbf{x}}_{t+1} - \Pi_{\mathcal{F}_{t}^{*}}(\widehat{\mathbf{x}}_{t})\right\Vert^{2}
\geq \frac{\lambda_{\ell}}{2}\left\Vert\widehat{\mathbf{x}}_{t+1} - \Pi_{\mathcal{F}_{t}^{*}}(\widehat{\mathbf{x}}_{t+1})\right\Vert^{2},
\end{equation}
where the second inequality follows directly from the definition of the Euclidean projection operator (see \cref{notation}). Substituting \eqref{eq:main-result-6} into \eqref{eq:main-result-5} and rearranging yields
\begin{align}
\label{eq:main-result-7}
\frac{1}{2\eta}\left\Vert\Pi_{\mathcal{F}_{t}^{*}}(\widehat{\mathbf{x}}_{t}) - \widehat{\mathbf{x}}_{t}\right\Vert^{2}
&\geq \frac{1}{2\eta}\left\Vert\widehat{\mathbf{x}}_{t+1} - \Pi_{\mathcal{F}_{t}^{*}}(\widehat{\mathbf{x}}_{t})\right\Vert^{2} + \frac{\lambda_{\ell}}{2}\left\Vert\widehat{\mathbf{x}}_{t+1} - \Pi_{\mathcal{F}_{t}^{*}}(\widehat{\mathbf{x}}_{t+1})\right\Vert^{2} \\ \nonumber
&\geq \left(\frac{1}{2\eta} + \frac{\lambda_{\ell}}{2}\right)\left\Vert\widehat{\mathbf{x}}_{t+1} - \Pi_{\mathcal{F}_{t}^{*}}(\widehat{\mathbf{x}}_{t+1})\right\Vert^{2},
\end{align}
which completes the proof.

\subsection{Proof of \cref{omgd}}
\label{omgd-proof}

For convenience, let us define the quantity
\begin{equation}
\label{eq:omgd-proof-1}
\widehat{\mathbf{x}}_{t}^{*} \equiv \Pi_{\mathcal{F}_{t}^{*}}(\widehat{\mathbf{x}}_{t}),
\quad t \in [T].
\end{equation}
By virtue of Assumption \ref{ass:smoothness} and the Cauchy–Schwarz inequality, we have
\begin{align}
\label{eq:omgd-proof-2}
f_{t}(\widehat{\mathbf{x}}_{t}) - \min_{\mathbf{x} \in \mathcal{X}} f_{t}(\mathbf{x})
&= f_{t}(\widehat{\mathbf{x}}_{t}) - f_{t}(\widehat{\mathbf{x}}_{t}^{*}) \\ \nonumber
&\leq \nabla f_{t}(\widehat{\mathbf{x}}_{t}^{*})^{\intercal}(\widehat{\mathbf{x}}_{t} - \widehat{\mathbf{x}}_{t}^{*}) + \frac{L_{f}}{2}\left\Vert\widehat{\mathbf{x}}_{t} - \widehat{\mathbf{x}}_{t}^{*}\right\Vert^{2} \\ \nonumber
&\leq \left\Vert\nabla f_{t}(\widehat{\mathbf{x}}_{t}^{*})\right\Vert\left\Vert\widehat{\mathbf{x}}_{t} - \widehat{\mathbf{x}}_{t}^{*}\right\Vert + \frac{L_{f}}{2}\left\Vert\widehat{\mathbf{x}}_{t} - \widehat{\mathbf{x}}_{t}^{*}\right\Vert^{2}.
\end{align}
On the other hand, for any $\alpha > 0$,
\begin{align}
\label{eq:omgd-proof-3}
0
&\leq \left(\frac{1}{\sqrt{2\alpha}}\left\Vert\nabla f_{t}(\widehat{\mathbf{x}}_{t}^{*})\right\Vert - \sqrt{\frac{\alpha}{2}}\left\Vert\widehat{\mathbf{x}}_{t} - \widehat{\mathbf{x}}_{t}^{*}\right\Vert\right)^{2} \\ \nonumber
&= \frac{1}{2\alpha}\left\Vert\nabla f_{t}(\widehat{\mathbf{x}}_{t}^{*})\right\Vert^{2}
  + \frac{\alpha}{2}\left\Vert\widehat{\mathbf{x}}_{t} - \widehat{\mathbf{x}}_{t}^{*}\right\Vert^{2}
  - \left\Vert\nabla f_{t}(\widehat{\mathbf{x}}_{t}^{*})\right\Vert\left\Vert\widehat{\mathbf{x}}_{t} - \widehat{\mathbf{x}}_{t}^{*}\right\Vert.
\end{align}
Rearranging yields
\begin{equation}
\label{eq:omgd-proof-4}
\left\Vert\nabla f_{t}(\widehat{\mathbf{x}}_{t}^{*})\right\Vert\left\Vert\widehat{\mathbf{x}}_{t} - \widehat{\mathbf{x}}_{t}^{*}\right\Vert
\leq \frac{1}{2\alpha}\left\Vert\nabla f_{t}(\widehat{\mathbf{x}}_{t}^{*})\right\Vert^{2}
	 + \frac{\alpha}{2}\left\Vert\widehat{\mathbf{x}}_{t} - \widehat{\mathbf{x}}_{t}^{*}\right\Vert^{2}.
\end{equation}
Substituting \eqref{eq:omgd-proof-4} into \eqref{eq:omgd-proof-2}, we obtain
\begin{equation}
\label{eq:omgd-proof-5}
f_{t}(\widehat{\mathbf{x}}_{t}) - \min_{\mathbf{x} \in \mathcal{X}} f_{t}(\mathbf{x})
\leq \frac{1}{2\alpha}\left\Vert\nabla f_{t}(\widehat{\mathbf{x}}_{t}^{*})\right\Vert^{2}
	 + \frac{L_{f}+\alpha}{2}\left\Vert\widehat{\mathbf{x}}_{t} - \widehat{\mathbf{x}}_{t}^{*}\right\Vert^{2}.
\end{equation}
Summing the above inequality over $t \in [T]$ then gives
\begin{align}
\label{eq:omgd-proof-6}
\sum_{t=1}^{T} f_{t}(\widehat{\mathbf{x}}_{t}) - \sum_{t=1}^{T}\min_{\mathbf{x} \in \mathcal{X}} f_{t}(\mathbf{x})
&\leq \frac{1}{2\alpha}\sum_{t=1}^{T}\left\Vert\nabla f_{t}(\widehat{\mathbf{x}}_{t}^{*})\right\Vert^{2}
	 + \frac{L_{f}+\alpha}{2}\sum_{t=1}^{T}\left\Vert\widehat{\mathbf{x}}_{t} - \widehat{\mathbf{x}}_{t}^{*}\right\Vert^{2} \\ \nonumber
&\leq \frac{G_{T}^{*}}{2\alpha}
	 + \frac{L_{f}+\alpha}{2}\sum_{t=1}^{T}\left\Vert\widehat{\mathbf{x}}_{t} - \widehat{\mathbf{x}}_{t}^{*}\right\Vert^{2},
\end{align}
for any $\alpha > 0$.

The next step is to bound the expression $\sum_{t=1}^{T}\Vert\widehat{\mathbf{x}}_{t} - \widehat{\mathbf{x}}_{t}^{*}\Vert^{2}$ from above. We begin by noting that
\begin{align}
\label{eq:omgd-proof-7}
\sum_{t=1}^{T}\left\Vert\widehat{\mathbf{x}}_{t} - \widehat{\mathbf{x}}_{t}^{*}\right\Vert^{2}
&\leq \left\Vert\widehat{\mathbf{x}}_{1} - \widehat{\mathbf{x}}_{1}^{*}\right\Vert^{2}
	  + 2\sum_{t=2}^{T}\left\Vert\widehat{\mathbf{x}}_{t} - \Pi_{\mathcal{F}_{t-1}^{*}}(\widehat{\mathbf{x}}_{t})\right\Vert^{2}
	  + 2\sum_{t=2}^{T}\left\Vert\Pi_{\mathcal{F}_{t}^{*}}(\widehat{\mathbf{x}}_{t}) - \Pi_{\mathcal{F}_{t-1}^{*}}(\widehat{\mathbf{x}}_{t})\right\Vert^{2} \\ \nonumber
&\leq \left\Vert\widehat{\mathbf{x}}_{1} - \widehat{\mathbf{x}}_{1}^{*}\right\Vert^{2}
	  + 2\sum_{t=2}^{T}\left\Vert\widehat{\mathbf{x}}_{t} - \Pi_{\mathcal{F}_{t-1}^{*}}(\widehat{\mathbf{x}}_{t})\right\Vert^{2}
	  + 2\mathcal{S}_{T}^{*},
\end{align}
where the first inequality follows from the triangle inequality, while the second follows from the definition of $\mathcal{S}_{T}^{*}$ in \eqref{eq:generalized-path-length-measures}. Next, by applying Lemma \ref{main-result} to the OMGD updates in \eqref{eq:omgd-update}, we obtain
\begin{equation}
\label{eq:omgd-proof-8}
\left\Vert\widehat{\mathbf{z}}_{t-1}^{(i+1)} - \Pi_{\mathcal{F}_{t-1}^{*}}(\widehat{\mathbf{z}}_{t-1}^{(i+1)})\right\Vert^{2}
\leq \left(1 - \frac{\lambda_{\ell}}{\lambda_{\ell} \ + \ \eta^{-1}}\right)\left\Vert\widehat{\mathbf{z}}_{t-1}^{(i)} - \Pi_{\mathcal{F}_{t-1}^{*}}(\widehat{\mathbf{z}}_{t-1}^{(i)})\right\Vert^{2},
\quad i \in [m],
\end{equation}
whence
\begin{align}
\label{eq:omgd-proof-9}
\left\Vert\widehat{\mathbf{x}}_{t} - \Pi_{\mathcal{F}_{t-1}^{*}}(\widehat{\mathbf{x}}_{t})\right\Vert^{2}
&= \left\Vert\widehat{\mathbf{z}}_{t-1}^{(m+1)} - \Pi_{\mathcal{F}_{t-1}^{*}}(\widehat{\mathbf{z}}_{t-1}^{(m+1)})\right\Vert^{2} \\ \nonumber
&\leq \left(1 - \frac{\lambda_{\ell}}{\lambda_{\ell} \ + \ \eta^{-1}}\right)^{m}\left\Vert\widehat{\mathbf{x}}_{t-1} - \Pi_{\mathcal{F}_{t-1}^{*}}(\widehat{\mathbf{x}}_{t-1})\right\Vert^{2}  \\ \nonumber
&\leq \frac{1}{4}\left\Vert\widehat{\mathbf{x}}_{t-1} - \widehat{\mathbf{x}}_{t-1}^{*}\right\Vert^{2},
\end{align}
where the second inequality is due to our choice of $m = \lceil\frac{\lambda_{\ell} \ + \ \eta^{-1}}{\lambda_{\ell}} \mathrm{ln}(4)\rceil$ and the definition \eqref{eq:omgd-proof-1}. Combining \eqref{eq:omgd-proof-7} and \eqref{eq:omgd-proof-9} gives
\begin{align}
\label{eq:omgd-proof-10}
\sum_{t=1}^{T}\left\Vert\widehat{\mathbf{x}}_{t} - \widehat{\mathbf{x}}_{t}^{*}\right\Vert^{2}
&\leq \left\Vert\widehat{\mathbf{x}}_{1} - \widehat{\mathbf{x}}_{1}^{*}\right\Vert^{2}
	  + \frac{1}{2}\sum_{t=2}^{T}\left\Vert\widehat{\mathbf{x}}_{t-1} - \widehat{\mathbf{x}}_{t-1}^{*}\right\Vert^{2}
	  + 2\mathcal{S}_{T}^{*} \\ \nonumber
&\leq \left\Vert\widehat{\mathbf{x}}_{1} - \widehat{\mathbf{x}}_{1}^{*}\right\Vert^{2}
	  + \frac{1}{2}\sum_{t=1}^{T}\left\Vert\widehat{\mathbf{x}}_{t} - \widehat{\mathbf{x}}_{t}^{*}\right\Vert^{2}
	  + 2\mathcal{S}_{T}^{*}.
\end{align}
Rearranging terms, we obtain
\begin{equation}
\label{eq:omgd-proof-11}
\sum_{t=1}^{T}\left\Vert\widehat{\mathbf{x}}_{t} - \widehat{\mathbf{x}}_{t}^{*}\right\Vert^{2}
\leq 4\mathcal{S}_{T}^{*} + 2\left\Vert\widehat{\mathbf{x}}_{1} - \widehat{\mathbf{x}}_{1}^{*}\right\Vert^{2}.
\end{equation}
Using this result in \eqref{eq:omgd-proof-6}, we conclude that
\begin{equation}
\label{eq:omgd-proof-12}
\sum_{t=1}^{T} f_{t}(\widehat{\mathbf{x}}_{t}) - \sum_{t=1}^{T}\min_{\mathbf{x} \in \mathcal{X}} f_{t}(\mathbf{x})
\leq \frac{G_{T}^{*}}{2\alpha} + 2(L_{f}+\alpha)\mathcal{S}_{T}^{*} + (L_{f}+\alpha)\left\Vert\widehat{\mathbf{x}}_{1} - \widehat{\mathbf{x}}_{1}^{*}\right\Vert^{2},
\quad \forall \alpha > 0.
\end{equation}

Finally, we show that the dynamic regret can still be bounded by $\mathcal{P}_{T}^{*}$. From \eqref{eq:omgd-proof-9}, we have
\begin{equation}
\label{eq:omgd-proof-13}
\left\Vert\widehat{\mathbf{x}}_{t} - \Pi_{\mathcal{F}_{t-1}^{*}}(\widehat{\mathbf{x}}_{t})\right\Vert
\leq \frac{1}{2}\left\Vert\widehat{\mathbf{x}}_{t-1} - \widehat{\mathbf{x}}_{t-1}^{*}\right\Vert.
\end{equation}
We can then choose the uniclass loss function such that $\rho = 1/2$ in \cref{main-result-corollary} to obtain
\begin{equation}
\label{eq:omgd-proof-14}
\sum_{t=1}^{T} f_{t}(\widehat{\mathbf{x}}_{t}) - \sum_{t=1}^{T} \min_{\mathbf{x} \in \mathcal{X}} f_{t}(\mathbf{x})
\leq 2K_{f}\mathcal{P}_{T}^{*} + 2K_{f}\left\Vert\widehat{\mathbf{x}}_{1} - \widehat{\mathbf{x}}_{1}^{*}\right\Vert,
\end{equation}
which completes the proof.

\end{document}